\documentclass[letterpaper, 10 pt, conference]{ieeeconf}

\IEEEoverridecommandlockouts
\usepackage{amsmath}
\usepackage{amssymb}
\usepackage{graphicx}
\usepackage{booktabs}
\usepackage{cite}
\usepackage{xcolor}
\usepackage{hyperref}
\usepackage{fancyhdr}
\hypersetup{colorlinks=true, linkcolor=black, citecolor=black, urlcolor=black}


\newtheorem{theorem}{Theorem}
\newtheorem{corollary}{Corollary}
\newtheorem{remark}{Remark}

\newcommand{\xk}{\boldsymbol{x}_k}
\newcommand{\xkk}{\boldsymbol{x}_{k+1}}
\newcommand{\uk}{\boldsymbol{u}_k}
\newcommand{\taua}{\tau^{a}_{k}}
\newcommand{\taud}{\tau^{d}_{k}}
\newcommand{\barta}{\bar{\tau}^{a}_{k}}
\newcommand{\V}{V}
\newcommand{\Vscale}{V_{\mathrm{scale}}}
\newcommand{\att}{\mathrm{att}}
\newcommand{\MSI}{\mathrm{MSI}}
\newcommand{\msictrl}{\mathrm{MSI}_{\mathrm{ctrl}}}
\newcommand{\rhojam}{\rho_{\mathrm{jam}}}
\newcommand{\jtf}{\mathrm{jtf}}
\newcommand{\Tset}{\mathcal{T}}

\title{\LARGE \bf
Inverting Self-Triggered Control: Adversarial Reinforcement Learning for Sparse Denial-of-Service Attacks
}

\author{Adam Haroon$^{1}$, Erick J.\ Rodr\'iguez-Seda$^{2}$, Tristan Schuler$^{3}$, and Cody Fleming$^{1}$%
\thanks{$^{1}$A. Haroon and C. Fleming are with the Department of Mechanical
        Engineering, Iowa State University, Ames, IA, USA
        {\tt\small aharoon@iastate.edu}}%
\thanks{$^{2}$E. J. Rodr\'iguez-Seda is with the Department of Weapons,
        Robotics, and Control Engineering, United States Naval Academy,
        Annapolis, MD, USA}%
\thanks{$^{3}$T. Schuler is with the Navy Center for Applied Research in
        Artificial Intelligence (NCARAI), U.S.\ Naval Research Laboratory,
        Washington, D.C., USA}%
\thanks{The views expressed in this document are those of the author(s) and do not reflect the official policy or position of the U.S. Naval Academy, Department of the Navy, the Department of War, or the U.S. Government.}
}

\begin{document}

\fancypagestyle{firstpage}{
    \fancyhf{} 
    \fancyfoot[C]{DISTRIBUTION STATEMENT A: Approved for public release, distribution is unlimited.} 
    \renewcommand{\headrulewidth}{0pt}
}
\pagestyle{plain}

\maketitle
\thispagestyle{firstpage}
\pagestyle{empty}

\begin{abstract}
Self-triggered reinforcement learning control (RL-STC) learns the
sparsest control schedule that preserves Lyapunov-decreasing
stability under a Run-Time Assurance (RTA) override. We invert
this: an adversarial RL agent learns the sparsest jamming or Denial-of-Service (DoS) schedule
that destabilizes the closed loop, with a Lyapunov-increase
admissibility predicate mirroring the defender's safety
certificate. We prove a plant-property lower bound on the
minimum jam count required for an immediate hold-last medium-access-control 
adversary to force a crash against a self-triggered controller (STC)
satisfying a Lyapunov contract, and recover a certificate-level
analog of the consecutive-grouping optimality of prior
count-budget DoS scheduling as a corollary. This extends the
DoS-scheduling count-budget analysis from periodic and linear-time-invariant to
STC controllers. Empirically, we train against four fixed
defenders per plant (one Linear Quadratic Regulator (LQR) and three RL-STC) on Pendulum, CartPole,
and Quadrotor2D. The learned adversary is the only adversary
that crashes every defender on every plant at $100\%$: greedy
misses Quadrotor2D LQR on $42\%$ of episodes and periodic misses
Pendulum LQR on $97\%$. On jam-time-per-failure it beats
baselines by up to $2.8\times$, and shows its widest absolute margin on
Quadrotor2D LQR. Robustness ablations show that Gaussian observation noise
exceeding the initial-state magnitude and position-only observation
both preserve $100\%$ failure rate and keep the learned adversary
strictly ahead of both baselines on jam-time-per-failure.
\end{abstract}

\section{Introduction}

A Denial-of-Service (DoS) adversary on a networked cyber-physical
system destabilizes the plant by withholding control updates
alone, without ever physically touching it. Certifying such
systems requires understanding the most resource-efficient
adversary, not only the worst-case continuous jammer. On a
shared wireless channel, communication is a scarce resource on
both sides: the defender expends sensing, computation, and
bandwidth on every decision; the adversary expends presence and
energy on every jam. The Self-triggered Reinforcement Learning
Control (RL-STC) framework of \cite{haroon2026learning}
produces a communication-efficient defender by learning the
sparsest transmission schedule that preserves Lyapunov-decreasing
stability. This paper asks the symmetric question: what does a
communication-efficient adversary look like when we invert the
same self-triggered structure onto the attack side?

Optimal-DoS-scheduling theory addresses a related question for
periodic and linear-time-invariant (LTI) controllers.
Prior work gives consecutive-grouping optimality under count-budget and
packet-dropping models \cite{zhang2015optimal, qin2018optimal};
\cite{yang2024optimal} extends this to time-varying interference.
A parallel line studies stability under frequency- and
duration-limited DoS \cite{wakaiki2020resilient,
feng2021jointly, defrancesco2015dos, cetinkaya2019stochastic}.
Recent Reinforcement Learning (RL) based work game-theorizes
attacker-defender interactions for Linear Quadratic Regulator
(LQR) control \cite{xing2025denial} and applies deep RL to
multi-channel jamming of remote state estimation and
event-triggered control \cite{xue2022jamming, hou2022deep,
geng2023security}. Surveys cover the
wireless-jamming taxonomy \cite{pirayesh2022jamming} and the
broader RL for cyber-physical systems (CPS) security landscape
\cite{huang2022reinforcement, cetinkaya2019analysis}. Across
this work the adversary's schedule is fixed, grid-scheduled, or
energy-budgeted; none of it produces a self-triggered adversary
against a self-triggered defender.

This paper contributes that construction. We invert RL-STC's
self-triggered structure onto the adversary and produce the
resulting attacker in two coupled components.
\emph{Empirically}, a Deep Q-Network (DQN) adversary acts on
the same discrete grid as the RL-STC defender, with a per-step
reward that mirrors the defender's: Lyapunov-increase flag,
graded Lyapunov term, jam-activity penalty $w_a$, and crash
bonus. A Lyapunov-increase admissibility predicate provides the
structural inverse of the Run-Time Assurance (RTA) certificate
of \cite[Prop.~1]{haroon2026learning}. We train against four
fixed defenders per plant (one LQR and three RL-STC at
Pareto-spread $w_c$ values) on Pendulum, CartPole, and
Quadrotor2D, sweeping $w_a$ across two orders of magnitude.
\emph{Theoretically}, we prove a plant-property lower bound
(Theorem~\ref{thm:bridge}) on the minimum jam count required
for an immediate hold-last Medium Access Control (MAC) adversary
to force a crash against any Self-Triggered Controller (STC)
satisfying a Lyapunov contract, and recover a certificate-level
analog of the consecutive-grouping optimality of
\cite{zhang2015optimal} as Corollary~\ref{cor:consecutive}. This
extends the count-budget analysis of
\cite{zhang2015optimal, qin2018optimal} from periodic and LTI
controllers to STC.

Three findings drive the rest of the paper. First,
the learned adversary is the only adversary that crashes every
defender on every plant at $100\%$: predicate-greedy misses
Quadrotor2D LQR on $42\%$ of episodes and periodic misses
Pendulum LQR on $97\%$. Second, on jam-time-per-failure
($\jtf$) the learned adversary strictly beats the best baseline
on all four CartPole defenders and on Quadrotor2D LQR
($2.6\times$ faster than either baseline, $2.7\times$ than greedy), and
matches greedy
on Pendulum where two-state dynamics leave predicate-greedy
near-optimal. Third, on Quadrotor2D under realistic
drone-adversary sensing constraints, the learned adversary
retains $100\%$ failure rate and strict $\jtf$ dominance under
Gaussian observation noise exceeding the initial-state magnitude and
under position-only observation.

\section{Problem Formulation}

\begin{figure}[t]
\centering
\includegraphics[width=\columnwidth]{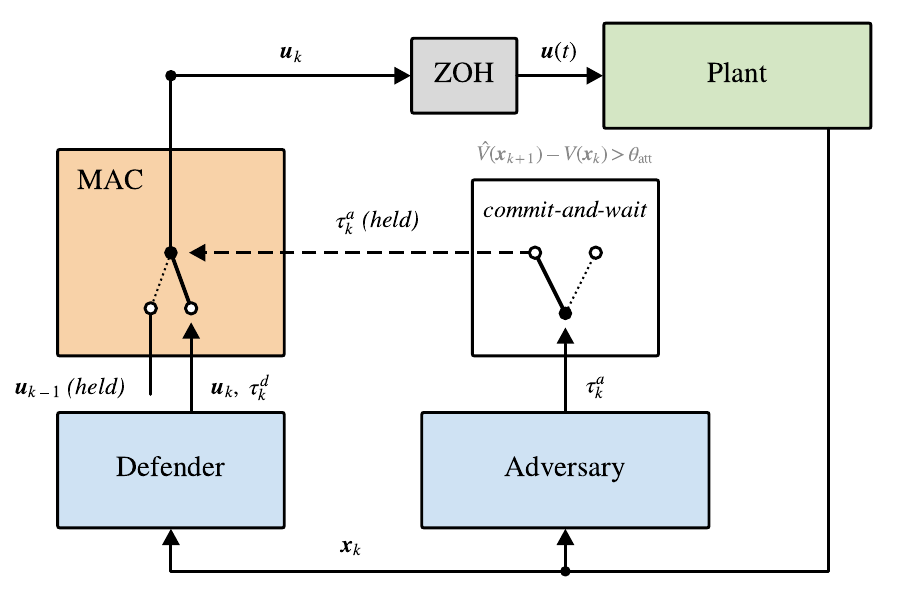}
\caption{Self-triggered \emph{adversary} framework, the structural
inverse of the RTA defender of \cite{haroon2026learning}. A fixed
defender outputs $(\uk,\taud)$; the learned adversary commits a jam
duration $\taua$ at each decision instant and holds it over the
resulting window. The Lyapunov-increase predicate
$\hat\V(\xkk)-\V(\xk) > \theta_{\att}$ is the adversary-side analog of
the defender's RTA certificate; it defines the predicate-greedy
baseline, while the learned adversary is left unshielded
(Sec.~\ref{ssec:reward}). Under the
hold-last MAC model a committed jam latches
$\boldsymbol{u}_{k-1}$ at the actuator, while $\taua = 0$ passes the
fresh update; both policies close the loop on the shared state
$\xk$.}
\label{fig:adversary_scheme}
\end{figure}

We adopt the plant, defender, and clock of~\cite{haroon2026learning} unchanged and add the adversary as
a symmetric self-triggered agent on the same discretization grid.
Fig.~\ref{fig:adversary_scheme} shows the resulting jammer framework.

\subsection{Nonlinear plant}

We consider a continuous, data-sampled, nonlinear dynamic system of the form
\begin{align}
    \dot{\boldsymbol x}(t) =& f(\boldsymbol x(t),\boldsymbol {u}_k) 
    \label{eq:nonlinear_system}
\end{align}
where $\boldsymbol x \in \mathbb{R}^n$ is the state vector, $\boldsymbol u \in{\cal U}\subset \mathbb{R}^m$ is the control input, and $f: \mathbb{R}^n\times {\cal U} \rightarrow \mathbb{R}^n $ is continuously differentiable. It is assumed that there is a bounded data-sample control law $\boldsymbol u(t)=\boldsymbol u_k:= \boldsymbol u(t_k)$ $\forall ~t\in [t_k,t_{k+1})$ capable of stabilizing the system around an equilibrium point $\boldsymbol x_{\mathrm{eq}}$.
The defender, adversary, and clock all act on this plant through the
sampled input $\boldsymbol u_k$.

\subsection{Defender (RL-STC baseline)}
\label{ssec:defender}

At decision instant $t_k$ the defender selects $(\uk, \taud)$
with $\taud \in \Tset = \{\tau_{\min}, 2\tau_{\min}, \dots,
N_{\tau}\tau_{\min}\}$ and $\tau_{\max} := N_{\tau}\tau_{\min}$, held by
zero-order hold over $[t_k, t_k + \taud)$; $N_{\tau}$ is the grid size, reserved throughout from the jam
count $N$ of Sec.~\ref{ssec:game}. An RTA layer overrides the action
with a clipped LQR backup $\uk \leftarrow \mathrm{clip}(-K\xk)$
at $\taud \leftarrow \tau_{\min}$ whenever the linearized
one-step safety prediction leaves the safe set.
\cite[Prop.~1]{haroon2026learning} certifies that an unsaturated
backup is Lyapunov-decreasing under the quadratic Lyapunov function
$\V(\boldsymbol x) = \boldsymbol x^\top P \boldsymbol x$, where $P$ solves the Continuous
Algebraic Riccati Equation (CARE) for the linearization $(A, B)$ of
\eqref{eq:nonlinear_system} with input weight $R$, and
$K = R^{-1} B^\top P$ is the corresponding LQR gain. The defender reward sums a
Lyapunov-decrease flag, a graded term $1 - \V(\xkk)/\Vscale$
($\Vscale$ fixed per plant), a communication term
$w_c((\MSI_k - \tau_{\min})/(\tau_{\max}-\tau_{\min}))^2$ on the
mean sampling interval $\MSI$ (the causal moving average of the
applied $\taud$), a penalty $r_{\mathrm{safe}}$ on RTA activation, and a crash
penalty of $-1000$. A \emph{crash} is the moment the plant's
state vector reaches a predefined unrecoverable unsafe set.

\subsection{Self-triggered adversary}
\label{ssec:reward}

The adversary selects a jam duration
$\taua \in \{0\} \cup \Tset$ at every commit-and-wait decision
instant. Over $[t_k, t_k + \taua)$ no control update reaches the
actuator, which holds the last applied input. $\taua = 0$ admits the
defender update at $t_k$ to pass unjammed; $\taua = \ell\tau_{\min}$
blocks the channel for $\ell$ ticks of the underlying grid before the
adversary decides again. The threat model is a MAC layer drop with a held actuator: the actuator latches the
most recent successfully delivered $\uk$ while jammed.

The adversary's per-step reward mirrors the defender's:
\begin{equation}
r^a_k = r^a_{\mathrm{stab}} + \frac{\V(\xkk)}{\Vscale}
        - w_a \left(\frac{\barta}{\tau_{\max}}\right)^2
        + r_{\mathrm{esc}}\,[\mathrm{terminated}],
\label{eq:adversary_reward}
\end{equation}
where $\barta$ is the causal exponential moving average
$\barta \leftarrow ((n_{\MSI}-1)\barta + \taua)/n_{\MSI}$ with
$n_{\MSI} = 5$ (reserved from the state dimension $n$) and
$\bar\tau^a_0 = 0$, the same recursion the defender's $\MSI_k$ uses in
\cite{haroon2026learning};
$r^a_{\mathrm{stab}} = +1$ if $\V(\xkk) > \V(\xk)$ and $-1$
otherwise; $r_{\mathrm{esc}}$ is the crash bonus. $w_a$ mirrors
$w_c$ and sweeps the adversary frontier. The admissibility
predicate of Sec.~\ref{ssec:predicate} is not enforced on the
learned adversary; a hard shield would forbid near-equilibrium
jams that lack a one-step Lyapunov increase but are needed to
drive the system out.

\subsection{Lyapunov-increase admissibility predicate}
\label{ssec:predicate}

The defender's RTA fires when the predicted safety scalar leaves
the safe set. The adversary-side analog admits a jam only if a
one-tick probe under the held input shows destabilization:
\begin{equation}
\hat\V(\xkk) - \V(\xk) > \theta_{\att},
\label{eq:predicate}
\end{equation}
where $\hat\V(\xkk) = \hat{\xkk}^\top P \hat{\xkk}$ and $\hat{\xkk}$
integrates the plant's nonlinear dynamics forward over a single
$\tau_{\min}$ under $\boldsymbol u_{\mathrm{last}}$, on the same
semi-implicit Euler substep the simulator uses. The probe is one tick
deep whatever duration the adversary then commits.
This is the structural inverse of Proposition~1's conclusion in
\cite{haroon2026learning}. We fix $\theta_{\att} = 0$
throughout, so any predicted Lyapunov increase passes.

\subsection{Two-player game and metrics}
\label{ssec:game}

Defender and adversary share $\xk$ and play $(\uk, \taud, \taua)$
with rewards $r^d_k, r^a_k$. Fixing the defender at one of four
representative policies reduces the co-trained game to a
single-player Markov Decision Process (MDP) for the adversary.
Over an episode of $T_{\mathrm{ep}}$ grid ticks with block
indicators $b_k \in \{0, 1\}$ and jam count $N = \sum_k b_k$, we
report: $\msictrl$ (defender's nominal mean inter-sample
interval, in seconds); $\rhojam = N/T_{\mathrm{ep}}$ (fraction of
episode ticks blocked); and $\jtf = \tau_{\min} N$ averaged over
crashing episodes.
$\rhojam^{\mathrm{fail}}$ restricts $\rhojam$ to crashing
episodes.

\begin{remark}[Jam-time / attack-energy identity]
\label{rem:energy}
Under the hold-last MAC model, a constant-power jammer satisfies
$\mathrm{energy} = P_{\mathrm{jam}} \tau_{\min} N$ (with
$P_{\mathrm{jam}}$ the transmit power, distinct from the Lyapunov
matrix $P$): jam time, jam fraction, and expended energy are
affine in $N$. $N$ is thus the count budget
of \cite{zhang2015optimal, qin2018optimal} up to a per-jam energy
constant, so $\jtf$ is attack energy per successful
destabilization and the $w_a$ sweep is a Lagrangian traversal of
that budget.
\end{remark}

\section{Minimum-Jam-Count Bound for STC}
\label{ssec:bridge}

Remark~\ref{rem:energy} makes $N$ the count budget of
\cite{zhang2015optimal, qin2018optimal}, but their optimality
results assume periodic or LTI controllers. We bridge to STC:
for any STC controller (\cite{wang2011stc} introduced the
self-triggered paradigm; \cite{heemels2012stc} gives the
exponential-decay Lyapunov contract we use, their eq.~14),
we derive a plant-property lower bound on $N$ under a hold-last
MAC adversary.

\subsection{Lower bound and corollaries}

Fix the plant with linearization $\dot{\boldsymbol x} = A \boldsymbol x + B \boldsymbol u$ around an
equilibrium $(\boldsymbol x_{\mathrm{eq}}, \boldsymbol u_{\mathrm{eq}})$ satisfying
$A \boldsymbol x_{\mathrm{eq}} + B \boldsymbol u_{\mathrm{eq}} = \boldsymbol 0$, safety Lyapunov
$V(\boldsymbol x) = (\boldsymbol x - \boldsymbol x_{\mathrm{eq}})^\top P (\boldsymbol x - \boldsymbol x_{\mathrm{eq}})$ with
$P \succ 0$, and safety threshold $V_{\mathrm{crit}}$ such that
the safety region is $\{\boldsymbol x : V(\boldsymbol x) \le V_{\mathrm{crit}}\}$.
Since $\boldsymbol x_{\mathrm{eq}} = \boldsymbol 0$ on all three plants, this $V$ is
the $\V$ of Sec.~\ref{ssec:defender}.
Let $\Phi(\tau) := e^{A\tau}$ be the transition matrix of the
autonomous linearization; when held-last input equals
$\boldsymbol u_{\mathrm{eq}}$ (which any Lyapunov-stabilizing controller
supplies at the equilibrium), the jam-window dynamics reduce to
$\dot{(\delta \boldsymbol x)} = A\, \delta \boldsymbol x$ with
$\delta \boldsymbol x := \boldsymbol x - \boldsymbol x_{\mathrm{eq}}$, and
$\delta \boldsymbol x(t+\tau) = \Phi(\tau)\, \delta \boldsymbol x(t)$. Non-equilibrium
held inputs add a bounded affine translation that we absorb into
the empirical margin of Table~\ref{tab:bridge-validation}.

Define the per-tick growth factor in the $P$-metric:
\begin{equation}
G_1 \;:=\;
\sup_{\boldsymbol x \ne \boldsymbol 0}
\sqrt{\frac{\boldsymbol x^{\!\top}\!\Phi(\tau_{\min})^{\!\top}\! P\, \Phi(\tau_{\min})\, \boldsymbol x}
           {\boldsymbol x^{\!\top}\! P\, \boldsymbol x}}.
\label{eq:G1}
\end{equation}
$G_1$ is a plant property depending only on $A$, $P$, and
$\tau_{\min}$: the operator norm of $\Phi(\tau_{\min})$ in the
$P$-norm. It handles both exponentially growing plants
($G_1 \sim \exp(\alpha \tau_{\min})$ with $\alpha$ the dominant
unstable-mode rate) and polynomially growing ones such as
integrator chains ($G_1 > 1$ from Jordan-block transient
amplification at spectral radius $1$).

Thm.~\ref{thm:bridge} below bounds an \emph{immediate}
adversary whose jam sequence begins at $t = 0$ from the plant's
reset distribution, without pre-conditioning on observed closed-loop
state; Sec.~\ref{sec:learned-vs-immediate} compares the bound, and that
adversary's measured threshold, to the learned adversary.

\begin{theorem}[Immediate-jam crash-count bound]
\label{thm:bridge}
Represent the plant of \eqref{eq:nonlinear_system} over a jam window by
its linearization about the equilibrium (fixed above; the
experiments of Sec.~\ref{sec:experiments} simulate the full nonlinear
dynamics). Let an STC controller satisfy the exponential-decay Lyapunov
contract $V(\hat{\boldsymbol x}(t_k + \tau^d)) \le V(\boldsymbol x_k)\exp(-\lambda
\tau^d)$ at each unjammed decision (as in
\cite[eq.~14]{heemels2012stc}), and consider a hold-last MAC
adversary whose first jam decision occurs at $t = 0$ from initial
state $\boldsymbol x_0$ with $V(\boldsymbol x_0) = V_0 < V_{\mathrm{crit}}$. Let $G_1$ be
defined by \eqref{eq:G1} and assume $G_1 > 1$ (the held-input
mode is $P$-expansive; otherwise no finite jam count crashes the
plant). Then the minimum tick count $N^{\star}$ in a single
contiguous jam block required to drive $V(\boldsymbol x_T) \ge
V_{\mathrm{crit}}$ satisfies
\begin{equation}
N^{\star} \;\ge\; \left\lceil
\frac{\log\!\left(V_{\mathrm{crit}} / V_0\right)}
     {2\,\log G_1}
\right\rceil.
\label{eq:bridge-bound}
\end{equation}
\end{theorem}

\begin{proof}
Working in $\delta \boldsymbol x$-coordinates, held-last input at
$\boldsymbol u_{\mathrm{eq}}$ gives autonomous jam-window dynamics
$\delta \boldsymbol x_T = \Phi(N\tau_{\min})\, \delta \boldsymbol x_0 =
\Phi(\tau_{\min})^N \delta \boldsymbol x_0$. Submultiplicativity of the
$P$-metric operator norm yields
$\|\Phi(\tau_{\min})^N\|_P \le \|\Phi(\tau_{\min})\|_P^N = G_1^N$,
hence
\begin{equation*}
V(\boldsymbol x_T) \;=\; \delta \boldsymbol x_T^\top P\, \delta \boldsymbol x_T \;\le\;
G_1^{2N}\, \delta \boldsymbol x_0^\top P\, \delta \boldsymbol x_0 \;=\; V_0\, G_1^{2N}.
\end{equation*}
A crash requires $V(\boldsymbol x_T) \ge V_{\mathrm{crit}}$, which by the
display forces $V_0 G_1^{2N} \ge V_{\mathrm{crit}}$; since
$G_1 > 1$ gives $\log G_1 > 0$, solving for $N$ and taking the
ceiling yields \eqref{eq:bridge-bound}. The same holds for any
schedule of $N$ jam ticks, not only the contiguous one: unjammed
epochs contribute factors $\exp(-\lambda\tau^d) \le 1$ by the STC
contract, so $V(\boldsymbol x_T) \le V_0 G_1^{2N}$ regardless of schedule and
$V_0 G_1^{2N} \ge V_{\mathrm{crit}}$ stays necessary for a crash.
\end{proof}

\begin{remark}
\label{rem:affine}
Beyond the linearization, a non-equilibrium held input
$\boldsymbol u_{\mathrm{last}}$ adds the affine term
$(\int_0^{\tau}\Phi(\tau-s)\,ds)B(\boldsymbol u_{\mathrm{last}}-\boldsymbol u_{\mathrm{eq}})$
to the jam-window map; over a window of $N$ ticks this is a
bounded translation of norm $O(\|\boldsymbol u_{\mathrm{last}} -
\boldsymbol u_{\mathrm{eq}}\|)$, which shifts $V_0$ by a bounded amount
absorbed into the empirical margin of
Table~\ref{tab:bridge-validation}. The linearization is exact only to
first order in $\delta \boldsymbol x$, so \eqref{eq:bridge-bound} certifies the
linearized model while Table~\ref{tab:bridge-validation} measures the
full nonlinear dynamics; that margin is consistent with a small combined
higher-order and affine correction but does not isolate it.
\end{remark}

Two corollaries follow.

\begin{corollary}[Consecutive-jam optimality on STC]
\label{cor:consecutive}
For a schedule $\pi$ with $N$ jam ticks and unjammed epochs of
durations $\{\tau_j^d\}_{j=1}^{M}$, the certified terminal
Lyapunov bound is
$\overline V(\pi) := V_0\, G_1^{2N}\prod_{j=1}^{M}
e^{-\lambda\tau_j^d}$, i.e.\ $V^\pi(\boldsymbol x_T) \le \overline V(\pi)$.
Among all schedules of fixed count $N$, $\overline V$ is
maximized by the contiguous block ($M = 0$), with value
$V_0 G_1^{2N}$.
\end{corollary}

\begin{proof}
Write $\|\boldsymbol v\|_P := (\boldsymbol v^\top P \boldsymbol v)^{1/2}$, so
$V = \|\cdot\|_P^2$. Each jam tick applies the linear map
$\Phi(\tau_{\min})$, of $P$-operator norm $G_1$; each unjammed epoch
contracts $\|\cdot\|_P$ by $e^{-\lambda\tau_j^d/2}$ at the realized
state, which is all the composition needs and does not require the
closed-loop STC map to be linear. Chaining these $N + M$ inequalities
along the trajectory and squaring gives
$V^\pi(\boldsymbol x_T) \le V_0 G_1^{2N}
\prod_{j=1}^{M} e^{-\lambda\tau_j^d} = \overline V(\pi)$. Every dwell
time is drawn from $\Tset$, so $\tau_j^d \ge \tau_{\min} > 0$ and each
factor is strictly below $1$; hence
$\overline V(\pi) \le V_0 G_1^{2N} =
\overline V(\pi_{\mathrm{cont}})$, strict whenever $M \ge 1$.
\end{proof}

Corollary~\ref{cor:consecutive} extends the consecutive-grouping
optimality of \cite{zhang2015optimal, qin2018optimal} from
periodic and LTI controllers to STC at the level of the certified
growth bound: their proof requires the LTI closed-loop transition
matrix, whereas ours needs only the STC contract and the $P$-norm
growth factor $G_1$. The bound is order-independent because
submultiplicativity does not depend on segment order; it is
attained exactly when $\|\Phi(\tau_{\min})\|_P$ equals the spectral
radius (as when $\Phi(\tau_{\min})$ is $P$-normal), and the
$G_1$-based certificate is the worst case otherwise.

\begin{corollary}[Plant-property gradient]
\label{cor:gradient}
For fixed $V_0$ and $V_{\mathrm{crit}}$, the pre-ceiling bound
$\log(V_{\mathrm{crit}}/V_0)/(2\log G_1)$ is strictly decreasing
in $G_1$, so $N^{\star}$ is non-increasing in $G_1$.
\end{corollary}

\begin{proof}
$\log(V_{\mathrm{crit}}/V_0) > 0$ is fixed and $\log G_1$ is positive
and strictly increasing for $G_1 > 1$, so the ratio strictly
decreases; the ceiling is monotone, hence $N^{\star}$ is
non-increasing.
\end{proof}

Corollary~\ref{cor:gradient} is a \emph{ceteris paribus}
statement. Across plants $V_{\mathrm{crit}}/V_0$ also varies, so
$N^{\star}$ need not order by $G_1$ alone (Table~\ref{tab:bridge-validation}:
the Pendulum's larger $\log(V_{\mathrm{crit}}/V_0)$ gives it the largest
$N^{\star}$ despite the smallest $G_1$, and the CartPole's small ratio the
smallest $N^{\star}$ despite an intermediate $G_1$). The
$\jtf$ trend, however, is monotone across the three LQR points
($0.755 < 0.876 < 1.136$ s as $G_1$ decreases from $1.26$ to
$1.22$), a suggestive but small-$n$ signature of the gradient.

\subsection{Empirical validation}
Table~\ref{tab:bridge-validation} reports $G_1$ and predicted
$N^{\star}$ per plant against a controlled immediate-jam attacker
that begins at $t = 0$ from the plant's reset distribution and
applies a single contiguous jam block of $N$ ticks, matching the
theorem's assumption class. For each plant we sweep $N$ and
record the smallest $N$ yielding crash rate $\ge 95\%$ over $100$
evaluation episodes against the periodic LQR defender, which
satisfies the Lyapunov contract trivially at
$\tau^d = \tau_{\min}$. $V_0$ is the sample mean of $V(\boldsymbol x_0)$ over
$100$ resets; $V_{\mathrm{crit}}$ is the minimum $V$ over the
plant's safety-region boundary.

\begin{table}[t]
\caption{Empirical validation of Theorem~\ref{thm:bridge} against
an immediate consecutive-jam attacker on the periodic LQR
defender. }
\label{tab:bridge-validation}
\begin{minipage}{\linewidth}
\centering
\footnotesize
\setlength{\tabcolsep}{4pt}
\begin{tabular}{lcccc}
\toprule
Plant & $G_1$ & $\log\!\frac{V_{\mathrm{crit}}}{V_0}$ &
  $N^{\star}$ (pred.\footnote{All counts in $\tau_{\min}$ ticks. Predicted
$N^{\star}$ computed from \eqref{eq:bridge-bound}; $G_1$, $V_0$ and
$V_{\mathrm{crit}}$ are produced by
\texttt{stage1/compute\_bridge\_theory\_vals.py}. Measured $N$ is the
smallest jam count achieving $\ge 95\%$ crash rate over $100$
evaluation episodes; the Quadrotor2D scan reaches only $91\%$ at
$N=80$, where it was stopped.}) & $N$ (meas.$^a$, $95\%$) \\
\midrule
Pendulum    & $1.22$ & $3.29$ & $\phantom{0}9$ & $30$ \\
CartPole    & $1.23$ & $1.02$ & $\phantom{0}3$ & $30$ \\
Quadrotor2D & $1.26$ & $2.21$ & $\phantom{0}5$ & $>80$ \\
\bottomrule
\end{tabular}
\end{minipage}
\end{table}

The bound holds on every plant, with measured $N \ge N^{\star}$
throughout at a ratio of at least $3.3\times$ (the Quadrotor2D entry is
right-censored, so no upper end is measured). The two counts answer
different questions, so a ratio above one is expected rather than
slack. $N^{\star}$ is a \emph{floor} on the certified model: no jam
sequence of fewer ticks can carry $V$ from $V_0$ past
$V_{\mathrm{crit}}$, in any direction and under any schedule. The
measured $N$ is an \emph{operational threshold}: the count at which one
contiguous block crashes $95\%$ of episodes drawn from the reset
distribution. A worst case over all directions and schedules exceeds a
typical case under one schedule even for an exact model, and
Remark~\ref{rem:affine} addresses what the linearization itself adds. The
bound's use is that it is a plant property, computable before any
experiment.

\subsection{Learned adversary vs.\ the immediate-jam bound}
\label{sec:learned-vs-immediate}

The learned adversary of Sec.~\ref{sec:experiments} does not defer its
attack. Against the same LQR defender it fires on the first tick in
every episode on Pendulum and CartPole and commits one contiguous block;
on Quadrotor2D it fires at median tick~$1$, from a state whose Lyapunov
value is $0.87\,V_0$, so \eqref{eq:bridge-bound} applied at fire time
returns a floor one tick \emph{larger} than at reset. Its advantage is
therefore not a better starting state.

The immediate attacker of Table~\ref{tab:bridge-validation} needs $30$,
$30$ and $>80$ ticks ($1.50$, $1.20$ and $>3.20$~s) to crash $95\%$ of
episodes with one contiguous block; the learned adversary crashes on
$1.136$, $0.876$ and $0.755$~s of jam time, i.e.\ $22.7$, $21.9$ and
$18.9$ ticks on average. These are different statistics, a mean over
crashing episodes against the smallest fixed count reaching $95\%$, so
the $24$--$76\%$ gap they imply is mean-versus-threshold rather than
matched. Matched at the $95$th percentile the learned adversary needs
$34$, $35$ and $31$ ticks: on Pendulum and CartPole it is no
cheaper than jamming from $t=0$. There it reproduces the
predicate-greedy schedule of Sec.~\ref{ssec:baselines}, matching greedy's
tick count to within about one tick; what it adds on those plants is
reliability rather than efficiency (Sec.~\ref{sec:experiments}). The robust gap is entirely
on Quadrotor2D, where greedy also fires immediately in one block but needs
$2.7\times$ the jam time at $58\%$ reliability, while the learned adversary
waits a tick, splits its budget across $1.79$ blocks, and crashes every
episode. Multi-step schedule structure, not pre-observation, is what the
immediate-jam threshold leaves unpriced.

\section{Experiments}
\label{sec:experiments}

\subsection{Setup}
\label{ssec:setup}

Plants and defender checkpoints are reused from
\cite{haroon2026learning}: an inverted Pendulum ($n=2$, $m=1$), a
CartPole ($n=4$, $m=1$), and a planar Quadrotor2D ($n=6$, $m=2$), each
run at $\tau_{\min} = 0.04$~s ($0.05$~s on the Pendulum) with
$N_{\tau} = 8$,
and each carrying an angle or position bound as its unsafe set. For each plant we train
$7 \times 4 \times 3 = 84$ adversaries across
$w_a \in \{0.01, 0.1, 0.5, 1.0, 2.0, 5.0, 10.0\}$, four fixed
defenders (a periodic LQR controller, baseline B1 of
\cite{haroon2026learning}, and three RL-STC defenders at
$w_c \in \{1.0, 6.0, 14.0\}$, Pareto-spread points from its
eleven-value $w_c$ sweep), and three seeds. Each adversary is evaluated over 100 episodes
against the three non-learned baselines of
Sec.~\ref{ssec:baselines} and three learned seeds. Periodic and predicate-greedy are deterministic
given the defender; the learned column reports mean $\pm$
cross-seed std.

The primary metric is jam-time-per-failure, the mean
jammed seconds over crashing episodes; by
Remark~\ref{rem:energy} this is attack energy per successful
destabilization. The jam fraction on failures
$\rhojam^{\mathrm{fail}}$ is reported alongside as a secondary
metric. Because periodic sits at $\rhojam^{\mathrm{fail}} \in
[49\%, 82\%]$, jam fraction alone would conflate
structural sparsity with strategic timing. A learned adversary
is \emph{strategic} on a defender when its mean $\jtf$ is
strictly below \emph{both} baselines and its cross-seed std is
small relative to the smaller gap.

\subsection{Adversary training and reward calibration}

The defender set is identical across plants. The
adversary is a DQN over $\{0, \dots, N_{\tau}=8\}$ with learning rate
$10^{-3}$, replay buffer $10^6$, $\gamma = 0.99$, and an
episode-bounded budget of $25{,}000$; the best-per-step
checkpoint is typically reached within $\sim 10\%$ of the
budget.

The terminal bonus $r_{\mathrm{esc}}$ in
(\ref{eq:adversary_reward}) is the one term requiring explicit
calibration. Naively mirroring the defender's $-1000$ penalty
places the terminal at least two orders of magnitude above the
differential jam cost between full-jam and half-jam schedules,
so DQN Q-targets become terminal-dominated and the Bellman
residual on jam-or-not decisions is lost. We set
$r_{\mathrm{esc}} = 100$ across all plants and $w_a$ values:
still exceeds worst-case non-terminal accumulation, but no
longer dominates the per-decision gradient.
Every other term in (\ref{eq:adversary_reward}) is a direct
structural mirror of the defender's reward.

\subsection{Baselines}
\label{ssec:baselines}

We compare the learned adversary to three non-learned adversaries on
every (defender, $w_a$) combination, each evaluated for 100
episodes:

\textit{NoJam.} A no-adversary control ($\taua = 0$) measuring
the defender's nominal behavior in isolation.

\textit{Periodic burst.} A state-independent jammer playing
$\taua = 4\tau_{\min}$ every other decision, giving
$\rhojam^{\mathrm{fail}} \in [49\%, 82\%]$. Non-learned analog
of the fixed-rate schedules contrasted with the grouped optimal
of \cite{zhang2015optimal, qin2018optimal}.

\textit{Predicate-greedy.} A model-based adversary committing the
longest available jam, $\taua = N_{\tau}\tau_{\min}$, whenever the
predicate (\ref{eq:predicate}) holds. It idles until the predicate first
fires and then jams continuously, giving jam fraction
$93$--$100\%$. It computes $\hat\V$ from the same state
$\xk$ the learned adversary observes, so learned-vs-greedy is a
same-information comparison isolating
the value of multi-step strategic timing over a one-step
predicate rule.

\subsection{Pendulum}

Table~\ref{tab:pendulum} reports the Pendulum frontier at
$w_a = 10.0$ (the same $w_a$ used throughout). All four
defenders are NoJam-self-stable ($0\%$ failure). The learned
adversary crashes every defender at $100\%$; predicate-greedy
misses LQR on $13\%$ of episodes and periodic misses on $97\%$,
making learned the sole reliable LQR crasher. On $\jtf$, learned
matches greedy within measurement precision ($1.06$--$1.14$~s
vs.\ $1.04$--$1.08$~s) and beats periodic by
$1.2$--$2.3\times$; on Pendulum, learned's $100\%$ headline
comes from reliability rather than strict $\jtf$ dominance,
since two-state dynamics leave predicate-greedy near-optimal.
On jam fraction, periodic sits at
$56$--$66\%$ on the three RL-STC defenders (sparse by construction; its
LQR figure rests on the $3\%$ of episodes it crashes) while learned and
greedy co-locate at $95$--$100\%$. The Pendulum panel of Fig.~\ref{fig:wa_frontier}
tracks this pattern across the $w_a$ sweep.

\begin{table}[t]
\caption{Pendulum frontier at $w_a = 10.0$: jam-time-per-failure
in seconds, mean $\pm$ cross-seed std over three seeds and $100$
evaluation episodes per seed.\textsuperscript{$\dagger$} No Pendulum
defender admits a \emph{strategic} learned adversary in the sense of
Sec.~\ref{ssec:setup} (learned never falls below greedy here), so
no row is bold; contrast Tables~\ref{tab:cartpole}
and~\ref{tab:quadrotor}.}
\label{tab:pendulum}
\centering
\footnotesize
\setlength{\tabcolsep}{4pt}
\begin{tabular}{lccc}
\toprule
defender & periodic & greedy & learned\\
\midrule
LQR (periodic)      & 1.400\textsuperscript{$\dagger$} & 1.076\textsuperscript{$\dagger$} & $1.136 \pm 0.068$ \\
RL-STC $w_c{=}1.0$  & 1.548 & 1.044 & $1.064 \pm 0.100$ \\
RL-STC $w_c{=}6.0$  & 1.986 & 1.076 & $1.119 \pm 0.050$ \\
RL-STC $w_c{=}14.0$ & 2.572 & 1.048 & $1.130 \pm 0.073$ \\
\bottomrule
\end{tabular}
\end{table}

Per-defender $\msictrl$: LQR $0.050$~s, $w_c \in \{1.0, 6.0,
14.0\}$ at $\{0.276, 0.398, 0.382\}$~s.
The continuous-jam optimum has a structural cause: under
held-actuator MAC drop, $\msictrl$ is irrelevant during a jam
window, so the cheapest crash blocks continuously until the
Pendulum falls. Learned and greedy both settle near this
lower bound.
\textit{$\dagger$Periodic and greedy fail to crash LQR reliably
(failure rate $3\%$ and $87\%$); all other cells achieve
$100\%$ failure.}

\subsection{CartPole}

Table~\ref{tab:cartpole} reports the CartPole frontier at
$w_a = 10.0$. Three of four defenders are self-stable; $w_c =
6.0$ falls on its own in $31\%$ of NoJam episodes, so its
learned-vs-baseline comparison is relative. Two of the four defenders admit a
\emph{strategic} learned adversary under the $\jtf$ gate; on the other
two it is faster than greedy but its cross-seed spread exceeds the
margin:

\textit{LQR:} learned $0.876 \pm 0.033$ s vs greedy $0.889$ s
($1.5\%$ faster). Periodic $4\tau_{\min}$ bursts fall inside the
LQR sampling-rate margin and never crash ($0\%$), so learned is
the only reliable LQR crasher.

\textit{$w_c = 1.0$:} learned $0.545 \pm 0.188$ s vs periodic
$0.592$, greedy $0.890$: $8\%$ / $39\%$ faster, jam fraction
$68.24 \pm 11.35\%$ (vs $69\%$ / $99\%$).

\textit{$w_c = 6.0$:} learned $0.320 \pm 0.000$ s matches
periodic $0.342$ with tighter variance and beats greedy $0.861$
by $2.7\times$.

\textit{$w_c = 14.0$:} learned $0.396 \pm 0.014$ s vs periodic
$0.606$, greedy $0.906$: $35\%$ / $2.3\times$ faster, jam
fraction $52.14 \pm 2.26\%$ (vs $49\%$ / $99\%$).

\begin{table}[t]
\caption{CartPole frontier at $w_a = 10.0$: jam-time-per-failure
in seconds, mean $\pm$ cross-seed std over three seeds and $100$
evaluation episodes per seed.\textsuperscript{$\dagger$} Bold marks a defender on which the learned adversary is \emph{strategic} in the sense of Sec.~\ref{ssec:setup}.}
\label{tab:cartpole}
\centering
\footnotesize
\setlength{\tabcolsep}{4pt}
\begin{tabular}{lccc}
\toprule
defender & periodic & greedy & learned\\
\midrule
LQR (periodic)      & --\textsuperscript{$\dagger$} & 0.889\textsuperscript{$\dagger$} & $0.876 \pm 0.033$ \\
RL-STC $w_c{=}1.0$  & 0.592 & 0.890 & $0.545 \pm 0.188$ \\
RL-STC $w_c{=}6.0$  & 0.342 & 0.861 & $\mathbf{0.320 \pm 0.000}$ \\
RL-STC $w_c{=}14.0$ & 0.606 & 0.906 & $\mathbf{0.396 \pm 0.014}$ \\
\bottomrule
\end{tabular}
\end{table}

Per-defender $\msictrl$: LQR $0.040$~s, $w_c \in \{1.0, 6.0,
14.0\}$ at $\{0.087, 0.201, 0.316\}$~s. NoJam failure rate is
$0\%$ except $w_c = 6.0$ ($31\%$, so its learned-vs-baseline
comparison is relative).
\textit{$\dagger$ Periodic fails to crash LQR (failure rate $0\%$) and
predicate-greedy reaches $95\%$ there; every other cell achieves $100\%$
failure.}

Fig.~\ref{fig:cartpole_traj} compares four CartPole episodes at
$w_c = 14.0$, $w_a = 10$: no-adversary $\theta$ stays near zero;
periodic bursts crash at $1.32$~s; predicate-greedy jams
continuously and crashes at $0.64$~s; the learned adversary
idles $\sim 0.24$~s while $\theta$ drifts under the defender,
then commits one sustained burst that crashes at $0.56$~s. The
learned policy is neither schedule-driven nor reactive but
state-timed.

The CartPole panel of Fig.~\ref{fig:wa_frontier} shows the
strategic plateau on $w_c = 14.0$ persisting across the sweep
and extending to $w_c = 1.0$; both stay below the greedy line
for every $w_a \ge 0.1$. LQR sits at $\rhojam^{\mathrm{fail}}
\approx 100\%$ (no exploitable timing structure) but learned
still crashes $1.5\%$ faster on jam time. $w_c = 6.0$ enters
the below-greedy regime from $w_a = 2$.

\begin{figure*}[t]
\centering
\includegraphics[width=\textwidth]{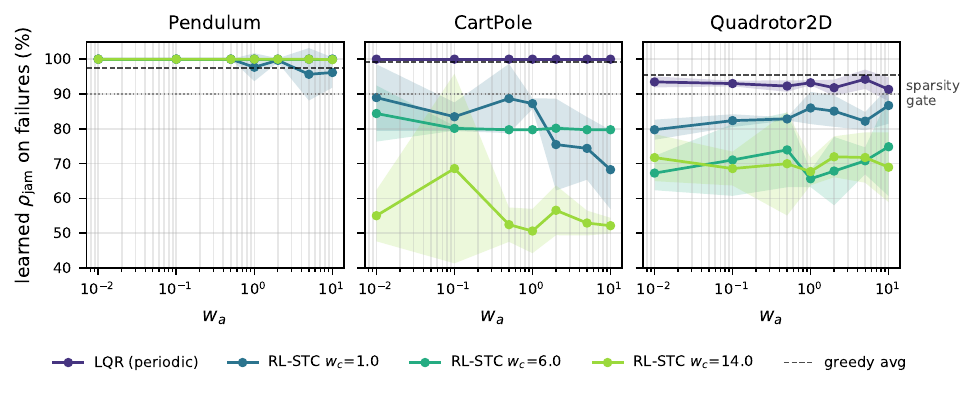}
\caption{Cross-plant $w_a$ frontier. Each panel shows learned
adversary jam fraction on crashing episodes versus jam-activity
weight, per defender, with shaded band the three-seed standard
deviation. The dashed line in each panel is the per-plant mean
greedy jam fraction; the dotted line is the $90\%$ sparsity gate.
Pendulum (left) keeps all defenders close to continuous jamming;
CartPole (middle) sparsifies all three RL-STC defenders below the gate
at $w_a = 10$ ($68$, $80$ and $52\%$); Quadrotor2D (right) sparsifies three RL-STC
defenders across most of the sweep and pulls even the LQR row
toward the gate.}
\label{fig:wa_frontier}
\end{figure*}

\begin{figure*}[t]
\centering
\includegraphics[width=\textwidth]{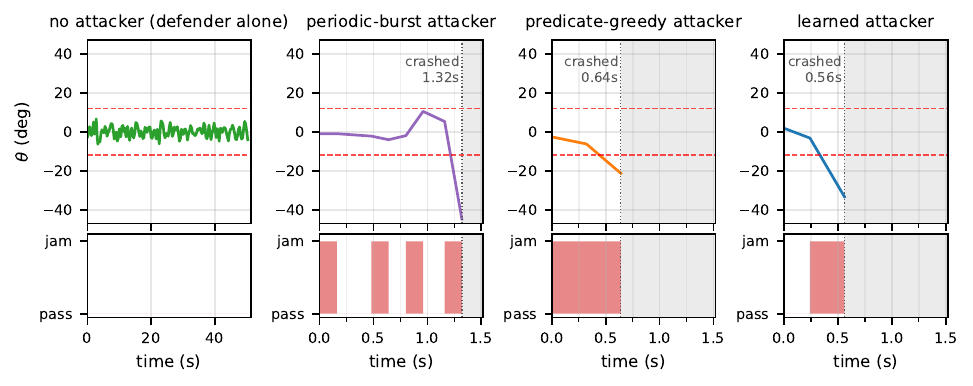}
\caption{Representative CartPole episodes at $w_c = 14.0$,
$w_a = 10$. Left to right: no adversary (stable over $50$ s);
periodic-burst (crash at $1.32$ s, $49\%$ jam fraction);
predicate-greedy (continuous jam, crash at $0.64$ s, $100\%$
fraction); learned (idle then sustained burst, crash at $0.56$ s,
$57\%$ fraction). Top row: $\theta$ vs time; dashed line is the
constraint $|\theta|_{\max} = 12^\circ$. Bottom row: jam mask.
Adversary panels share a time axis; shading marks post-crash
region. The no-adversary panel uses its own axis (episode two
orders of magnitude longer than any crash).}
\label{fig:cartpole_traj}
\end{figure*}

\subsection{Quadrotor2D}

Table~\ref{tab:quadrotor} reports the Quadrotor2D frontier at
$w_a = 10.0$. All four defenders are self-stable under NoJam
($0\%$ failure); two admit strategic learned adversaries. LQR
shows the largest learned-vs-baseline gap in the paper.

\textit{LQR:} learned $0.755 \pm 0.099$ s vs periodic $1.994$
and greedy $2.074$: $2.7\times$ faster than greedy, and the
only reliable LQR crasher ($100\%$ vs greedy $58\%$). Jam
fraction $91.31 \pm 3.41\%$ (periodic $82\%$, greedy $99.8\%$).

\textit{$w_c = 14.0$:} learned $0.654 \pm 0.195$ s vs periodic
$0.920$, greedy $1.830$: $29\%$ / $2.8\times$ faster, jam
fraction $68.95 \pm 10.01\%$. Its spread is $73\%$ of the margin to
periodic, the loosest of the four strategic cells (the others are under
$10\%$), so it is the weakest of the four.

\textit{$w_c \in \{1.0, 6.0\}$:} both learned adversaries crash
at $100\%$ but at $\jtf$ higher than periodic ($1.068$~s and
$1.237$~s vs.\ $0.602$~s and $0.637$~s). Their short
$\msictrl$ ($0.108$~s, $0.206$~s) keeps the plant close to
equilibrium, and the state-uniform reward shaping
(Sec.~\ref{ssec:reward}) trades strategic-timing performance
here for uniform crash reliability across the state space
(discussed in Sec.~\ref{sec:limitations}). Both cells still
beat greedy by $31$--$33\%$ on $\jtf$ and by $9$--$18$~pp on
jam fraction.

\begin{table}[t]
\caption{Quadrotor2D frontier at $w_a = 10.0$: jam-time-per-failure
in seconds, mean $\pm$ cross-seed std over three seeds and $100$
evaluation episodes per seed.\textsuperscript{$\dagger$} Bold marks a defender on which the learned adversary is \emph{strategic} in the sense of Sec.~\ref{ssec:setup}.}
\label{tab:quadrotor}
\centering
\footnotesize
\setlength{\tabcolsep}{4pt}
\begin{tabular}{lccc}
\toprule
defender & periodic & greedy & learned\\
\midrule
LQR (periodic)      & 1.994 & 2.074\textsuperscript{$\dagger$} & $\mathbf{0.755 \pm 0.099}$ \\
RL-STC $w_c{=}1.0$  & 0.602 & 1.558 & $1.068 \pm 0.868$ \\
RL-STC $w_c{=}6.0$  & 0.637 & 1.846 & $1.237 \pm 0.750$ \\
RL-STC $w_c{=}14.0$ & 0.920 & 1.830 & $\mathbf{0.654 \pm 0.195}$ \\
\bottomrule
\end{tabular}
\end{table}

Per-defender $\msictrl$: LQR $0.040$~s, $w_c \in \{1.0, 6.0,
14.0\}$ at $\{0.108, 0.206, 0.257\}$~s. NoJam failure rate is
$0\%$ across defenders.
\textit{$\dagger$ Predicate-greedy crashes LQR on only $58\%$
of episodes; learned achieves $100\%$ failure on every defender.}

The Quadrotor2D LQR result is the largest learned-vs-baseline
gap in the paper. The reason is multi-mode: termination fires on
any of $|\theta|$, $|x|$, or $|z|$ exceeding threshold, giving
three destabilization routes whose joint coverage admits timing
structure even against continuous-sampling LQR. Neither periodic
nor greedy exploits this structure. On short-$\msictrl$ RL-STC
defenders ($w_c \in \{1.0, 6.0\}$), periodic's burst cadence
happens to align with the defender's sampling, producing shorter
$\jtf$ than the learned adversary.

\subsection{Deploy-realism ablation (Quadrotor2D)}

The three plants above establish
cross-plant generality. We now stress-test the framework on
Quadrotor2D under the sensing constraints a real drone-jamming
adversary would face: noisy radio-frequency (RF) or visual state
estimates, and velocity-blind observation when the adversary
lacks derivative estimators. We re-train the learned adversary
on Quadrotor2D $w_c = 14.0$, $w_a = 10$ under two perturbations
mirroring these constraints: (i)~Gaussian noise
$\mathcal{N}(0, \sigma_{\mathrm{obs}}^2)$ per state component,
$\sigma_{\mathrm{obs}} \in \{0, 0.01, 0.05, 0.1, 0.2\}$, the largest
exceeding the mean magnitude of every state component at reset
($\le 0.17$); and
(ii)~position-only, masking velocity components
$(\dot x, \dot z, \dot\theta)$ to zero. Periodic is
observation-blind (invariant $\jtf$); predicate-greedy uses the
same perturbed observation as learned. Three training seeds per
configuration.

\begin{table}[t]
\caption{Quadrotor2D robustness ablation on $w_c = 14.0$,
$w_a = 10$, jam-time-per-failure in seconds. All adversaries see the
same perturbed observation channel; learned reports mean
$\pm$ cross-seed standard deviation over three training seeds.
Failure rate is $100\%$ in every row; bold marks the learned adversary
beating both baselines, which it does in all six.}
\label{tab:robustness}
\centering
\footnotesize
\setlength{\tabcolsep}{4pt}
\begin{tabular}{lccc}
\toprule
condition & periodic & greedy & learned\\
\midrule
clean (reference)              & 0.920 & 1.830 & $\mathbf{0.654 \pm 0.195}$ \\
$\sigma_{\mathrm{obs}} = 0.01$ & 0.920 & 1.856 & $\mathbf{0.532 \pm 0.097}$ \\
$\sigma_{\mathrm{obs}} = 0.05$ & 0.920 & 1.862 & $\mathbf{0.471 \pm 0.123}$ \\
$\sigma_{\mathrm{obs}} = 0.10$ & 0.920 & 1.856 & $\mathbf{0.606 \pm 0.047}$ \\
$\sigma_{\mathrm{obs}} = 0.20$ & 0.920 & 1.747 & $\mathbf{0.631 \pm 0.037}$ \\
position-only                  & 0.920 & 1.811 & $\mathbf{0.681 \pm 0.091}$ \\
\bottomrule
\end{tabular}
\end{table}

Failure rate stays at $100\%$ everywhere
(Table~\ref{tab:robustness}). Learned strictly beats both
baselines on $\jtf$ in all six conditions, holding at
$0.47$--$0.68$~s vs.\ periodic $0.92$~s and greedy
$1.75$--$1.86$~s, with cross-seed variance $\le 0.20$~s. The
framework remains deployable under realistic drone-adversary
sensing constraints.

\subsection{Discussion}

\textit{Plants stratify on reliability and strict $\jtf$
dominance.} Learned crashes all $12$ (plant, defender) cells at
$100\%$; baselines miss (greedy $42\%$ on Quadrotor2D LQR,
periodic $97\%$ on Pendulum LQR). On $\jtf$ learned strictly
beats the best baseline on $4/4$ CartPole cells and $2/4$
Quadrotor2D cells (largest gap $2.6\times$ on LQR, against its best
baseline) and matches
greedy on Pendulum. The stratification tracks constraint width
(single-scalar for Pendulum/CartPole vs.\ disjunctive for
Quadrotor2D) and near-equilibrium trajectory density
(short-$\msictrl$ RL-STC on Quadrotor2D has wider $\jtf$
variance; see the reward-shaping trade-off below).

\textit{Continuous jamming is the count-budget optimum against
conventional LQR.} \cite{zhang2015optimal, qin2018optimal} prove
consecutive-grouping optimality under count-budget and
packet-dropping models. Against periodic LQR, the learned
adversary settles into this continuous-jam optimum on Pendulum
and CartPole; only Quadrotor2D's disjunctive constraint breaks
the single-mode geometry and lets learned beat continuous.
Against RL-STC, the defender's learned timing exposes strategic
windows even on single-scalar plants (CartPole).

\textit{Learned timing dominates greedy on $\jtf$; defender
structure dominates $\msictrl$ in attackability.} Learned is strictly faster than greedy on all four CartPole
cells ($1.01$--$2.7\times$) and on Quadrotor2D LQR ($2.7\times$), and matches greedy within measurement precision on
Pendulum, where the two-state dynamics leave predicate-greedy
near-optimal. Against periodic it is faster on $9/12$ cells and
slower on the two short-$\msictrl$ Quadrotor2D cells ($w_c \in
\{1.0, 6.0\}$), where periodic's fixed cadence aligns with the
sampling rate. CartPole LQR is not comparable, since periodic
never crashes there. CartPole $w_c = 14.0$ stays
below greedy across $w_a \in [0.1, 10]$ while $w_c = 6.0$
enters the below-greedy regime only from $w_a = 2$, despite
similar $\msictrl$; the defender's internal policy structure,
not its nominal communication budget, drives what the adversary
can exploit.

\textit{Trade-off in state-uniform reward shaping.}
\label{sec:limitations}
$r^a_{\mathrm{stab}}$ applies uniformly across the state space
to mirror the defender's RTA. Near equilibrium the graded reward
$V(\xkk)/\Vscale$ is small while the $\pm 1$ shaping term still
fires on any drift, incentivizing continuous near-equilibrium
jamming. On the two Quadrotor2D cells with the largest bias
($w_c \in \{1.0, 6.0\}$), learned crashes at $100\%$ but at
$1.6\times$ the $\jtf$ and $5.5\times$ the cross-seed spread of that
plant's two strategic cells. A potential-based reward that
vanishes at safe equilibrium is left for future work. Three further
limits bound the scope. The count bound of Sec.~\ref{ssec:bridge} is
stated on the linearization, and Table~\ref{tab:bridge-validation}
measures the nonlinear plant (Remark~\ref{rem:affine}); the adversary is a single DQN
architecture evaluated on three plants; and on the Pendulum it is
reliable rather than strategic, matching predicate-greedy on $\jtf$.

\section{Conclusion and Future Work}
\label{sec:future}

We presented a self-triggered adversary that inverts the
RL-STC defender of \cite{haroon2026learning} and is
defender-agnostic: it treats any fixed defender as a
closed-loop black box and applies the same training pipeline.
Under the hold-last MAC model, jam time, attack energy, and jam
fraction are a single resource, and the $w_a$ sweep traverses
the energy budget of the optimal-DoS literature. We validated
the formulation on periodic LQR and three RL-STC defenders
across Pendulum, CartPole, and Quadrotor2D; on Quadrotor2D,
the framework retains $100\%$ failure rate and strict dominance
over predicate-greedy under realistic drone-adversary sensing
constraints. Two follow-ons are immediate.

\textit{Co-training as a two-player Markov game.} Replace the
fixed defender with an RL-STC agent learning concurrently
against the adversary; this tests whether the defender can
learn timing patterns the adversary cannot exploit.

\textit{Predicate as action mask.} The admissibility predicate
defines the greedy baseline and shapes the adversary reward but
is not enforced on the learned adversary. Hard-masking
inadmissible jams would mirror the defender's RTA shield and
isolate the sparsity gains from reward shaping (with the caveat
in Sec.~\ref{ssec:reward} that a hard mask forbids
equilibrium-time jams needed to drive the system out).


\bibliographystyle{IEEEtran}
\bibliography{refs_acc}

\end{document}